\documentclass[12pt]{article}
 
 \usepackage[page,header]{appendix}
\usepackage{titletoc}

\usepackage{todonotes}
\usepackage{graphicx}
\usepackage[numbers,sort]{natbib}
\usepackage{etex}
\usepackage[utf8]{inputenc}
\usepackage[dvips,letterpaper,margin=1in]{geometry}
\usepackage{amsmath}
\usepackage{amssymb}
\usepackage{amsthm}
\usepackage{bm}
\usepackage{colonequals}
\usepackage{comment}
\usepackage{graphicx}
\usepackage{subcaption}
\usepackage{xcolor}
\usepackage{tikz} 
\usetikzlibrary{shadows} 
\usepackage{authblk}
\usepackage[utf8]{inputenc} 
\usepackage[T1]{fontenc}    
\usepackage{url}            
\usepackage{booktabs}       
\usepackage{amsfonts}       
\usepackage{nicefrac}       
\usepackage{microtype}      
\usepackage{hyperref}
\hypersetup{colorlinks = true,
            linkcolor = blue,
            urlcolor  = teal,
            citecolor = magenta,
            anchorcolor = blue}

\usepackage{amsmath,amsthm,amssymb}
\usepackage{algorithm,algpseudocode}
\usepackage{wrapfig}
\usepackage{esint} 
\usepackage[margin=1in]{geometry} 
\usepackage{amsmath,amsthm,amssymb,bbm,mathtools}
\usepackage{xcolor}
\usepackage{ytableau}
\usepackage{caption}
\usepackage{subcaption}
\usepackage{enumerate}
\usepackage{enumitem}
\setlist{parsep = -0em, itemsep = 0.25em}

\usepackage{physics}

\usepackage{amsmath,amsthm,amssymb}
\usepackage{algorithm,algpseudocode}
\usepackage{wrapfig}
\usepackage{esint} 
\usepackage[margin=1in]{geometry} 
\usepackage{amsmath,amsthm,amssymb,bbm,mathtools}
\usepackage{xcolor}
\usepackage{ytableau}
\usepackage{caption}
\usepackage{subcaption}
\usepackage{bm}
\usepackage{physics}

\begin{document}

\title{Symmetric Single Index Learning}
\date{}
\author[a]{Aaron Zweig}
\author[a,b]{Joan Bruna}

\affil[a]{Courant Institute of Mathematical Sciences, New York
  University, New York}
\affil[b]{Center for Data Science, New York University}

\maketitle

%


\newcommand{\aaron}[1]{{\color{purple}[AZ: #1]}}
\newcommand{\joan}[1]{{\color{teal}[JB: #1]}}
\newcommand{\clay}[1]{{\color{red}[CS: #1]}}

\newcommand{\N}{\mathbb{N}}
\newcommand{\Z}{\mathbb{Z}}

\newcommand{\A}{\mathcal{A}}
\renewcommand{\AA}{\mathcal{A}_{0}}
\newcommand{\diff}{*}

\newcommand{\one}{\mathbf{1}}

\renewcommand{\P}{\mathcal{P}}
\newcommand{\X}{\mathcal{X}}
\newcommand{\F}{\mathcal{F}}
\renewcommand{\P}{\mathcal{P}}
\newcommand{\syml}{\text{Sym}_L}
\newcommand{\symlt}{\text{Sym}_L^2}
\newcommand{\proj}{\text{Proj}}
\newcommand{\EE}{\mathbb{E}}

\newcommand{\comp}{\eta}
\newcommand{\btheta}{\boldsymbol{\theta}}
\newcommand{\blamb}{\boldsymbol{\lambda}}
\newcommand{\balpha}{\boldsymbol{\alpha}}

\newcommand{\V}{V}
\renewcommand{\S}{{S^1}}

\newcommand{\M}{\mathcal{M}}
\renewcommand{\F}{\mathcal{F}}

\renewcommand{\i}{\mathbf{i}}
\newcommand{\ind}{\mathbbm{1}}
\newcommand{\tp}{p}
\newcommand{\bp}{\bold{p}}
\renewcommand{\bm}{\bold{m}}

\newcommand{\tbp}{\bold{p}}

\newcommand{\nice}{nice}

\newcommand*\conj[1]{\overline{#1}}


\newtheorem{theorem}{Theorem}[section]
\newtheorem{lemma}[theorem]{Lemma}
\newtheorem{proposition}[theorem]{Proposition}
\newtheorem{corollary}[theorem]{Corollary}
\newtheorem{definition}[theorem]{Definition}
\newtheorem{remark}[theorem]{Remark}
\newtheorem{assumption}[theorem]{Assumption}

\newcommand{\fix}{\marginpar{FIX}}
\newcommand{\new}{\marginpar{NEW}}


\begin{abstract}
Few neural architectures lend themselves to provable learning with gradient based methods. One popular model is the single-index model, in which labels are produced by composing an unknown linear projection with a possibly unknown scalar link function. Learning this model with SGD is relatively well-understood, whereby the so-called information exponent of the link function governs a polynomial sample complexity rate.  However, extending this analysis to deeper or more complicated architectures remains challenging.

In this work, we consider single index learning in the setting of symmetric neural networks.  Under analytic assumptions on the activation and maximum degree assumptions on the link function, we prove that gradient flow recovers the hidden planted direction, represented as a finitely supported vector in the feature space of power sum polynomials.  We characterize a notion of information exponent adapted to our setting that controls the efficiency of learning.
\end{abstract}

\section{Introduction}

Quantifying the advantage of neural networks over simpler learning systems remains a primary question in deep learning theory. Specifically, understanding their ability to discover relevant low-dimensional features out of high-dimensional inputs is a particularly important topic of study.  One facet of the challenge is explicitly characterizing the evolution of neural network weights through gradient-based methods, owing to the nonconvexity of the optimization landscape. 

The single index setting, long studied in economics and biostatistics~\citep{radchenko2015high} offers the simplest setting where non-linear feature learning can be characterized explicitly.  In this setting, functions of the form $x \mapsto f(\langle x, \theta^* \rangle)$ where $\theta^* \in \mathcal{S}_{d-1}$ represents a hidden direction in high-dimensional space, and $f$ a certain non-linear link function, are learned via a student with an identical architecture $x \mapsto f(\langle x, \theta \rangle)$, under certain data distribution assumptions, such as Gaussian data. Gradient flow and gradient descent~\citep{yehudai2020learning,arous2021online,dudeja2018learning} in this setting can be analyzed by reducing the high-dimensional dynamics of $\theta$ to dimension-free dynamics of appropriate \emph{summary statistics}, given in this case by the scalar correlation $\langle \theta, \theta^* \rangle$.

The efficiency of gradient methods in this setting, measured either in continuous time or independent samples, is controlled by two main properties. First, the correlation initialization, which typically scales as $\frac{1}{\sqrt{d}}$ for standard assumptions. Second, the information exponent $s_f$ of $f$ ~\citep{arous2021online, dudeja2018learning, biettilearning2022, damian2023smoothing,damian2022neural,abbe2023sgd}, which measures the number of effective vanishing moments of the link function --- leading to a sample complexity of the form $O(d^{s-1})$ for generic values of $s$. 

While this basic setup has been extended along certain directions, e.g. relaxing the structure on the input data distribution \cite{yehudai2020learning, bruna2023single}, considering the multi-index counterpart \cite{damian2022neural,abbe2022merged, abbe2023sgd, arnaboldi2023high}, or learning the link function with semi-parametric methods \cite{biettilearning2022,mahankali2023beyond}, they are all fundamentally associated with fully-connected shallow neural networks. Such architecture, for all its rich mathematical structure, also comes with important shortcomings. 
In particular, it is unable to account for predefined symmetries in the target function that the learner wishes to exploit. This requires specialized neural architectures enforcing particular invariances, setting up novel technical challenges to carry out the program outlined above. 

In this work, we consider arguably the easiest form of symmetry, given by permutation invariance. The primary architecture for this invariance is DeepSets~\citep{zaheer2017deep}, which is necessarily three layers by definition and therefore not a simple extension of the two layer setting. 
In order to quantify the notion of `symmetric' feature learning in this setting, we introduce a symmetric single index target, and analyze the ability of gradient descent over a DeepSet architecture to recover it. 
Under appropriate assumptions on the model, initialization and data distribution, we combine the previous analyses with tools from symmetric polynomial theory to characterize the dynamics of this learning problem.  Our primary theorem is a proof of efficient learning under gradient flow, with explicit polynomial convergence rates controlled by an analogue of information exponent adapted to the symmetric setting. Combined with other contemporary works, this result solidifies the remarkable ability of gradient descent to perform feature learning under a variety of high-dimensional learning problems.

\section{Setup}

\subsection{Notation}

For $z \in \mathbb{C}$, we will use $\overline{z}$ to denote the complex conjugate, with the notation $z^*$ always being reserved to denote a special value of $z$ rather than an operation.  For complex matrices $A$ we will use $A^\dag$ to denote the conjugate transpose.  The standard inner product on $\mathbb{C}^N$ is written as $\langle \cdot, \cdot \rangle$, whereas inner products on $L^2(\gamma)$ spaces for some probability measure $\gamma$ will be written as $\langle \cdot, \cdot \rangle_\gamma$.  Furthermore, for $h$ a vector and $p(x)$ a vector-valued function, we will use $\langle h, p \rangle_\gamma$ as shorthand for the notation $\langle h, p(\cdot) \rangle_\gamma$.

\subsection{Regression setting and Teacher function}

We consider a typical regression setting, where given samples $(x,y) \in \mathcal{X} \times \mathbb{C}$ with $y = F(x)$, we seek to learn a function $F_w$ with parameter $w \in \mathbb{C}^M$ by minimizing some expected loss $E_{x\sim\nu}\left[L(F(x), F_w(x)) \right]$.  Note that we consider complex-valued inputs and parameters because they greatly simplify the symmetric setting (see Proposition~\ref{prop:hermite}), hence we will also assume $\mathcal{X} \subseteq \mathbb{C}^N$. Both $F$ and $F_w$ will be permutation invariant functions, meaning that $F(x_{\pi(1)},\ldots x_{\pi(N)}) = F(x_1, \ldots, x_N)$ for any permutation $\pi:~\{1,N\}\to \{1,N\}$. 

Typically the single index setting assumes that the trained architecture will exactly match the true architecture (e.g. as in \cite{arous2021online}), but below we will see why it's necessary to consider separate architectures. For that reason, we'll consider separately defining the teacher $F$ and the student $F_w$.

The first ingredient are the power sum polynomials:
\begin{definition}
    For $k \in \N$ and $x \in \mathbb{C}^N$, the \emph{normalized powersum polynomial} is defined as $$p_k(x) = \frac{1}{\sqrt{k}} \sum_{n=1}^N x_n^k~.$$
\end{definition}

Let $p(x) = [p_1(x), p_2(x), \dots]$ be an infinite dimensional vector of powersums, and consider a fixed vector $h^* \in \mathbb{C}^\infty$ of unit norm.  Then our teacher function $F$ will be of the form
\begin{align}
    F : \mathcal{X} &\to \mathbb{C}\\
    x &\mapsto F(x) := f(\langle h^*, p(x)\rangle)
\end{align}
for some scalar link function $f: \mathbb{C} \to \mathbb{C}$. $F$ may thus be understood as a single-index function in the feature space of powersum polynomials.

\subsection{DeepSets Student Function}

Let us remind the typical structure of a DeepSets network~\citep{zaheer2017deep}, where for some maps $\Phi: \mathcal{X} \rightarrow \mathbb{C}^M$ and $\rho: \mathbb{C}^M \rightarrow \mathbb{C}$, the standard DeepSets architecture is of the form:
\begin{align}
    x \mapsto \rho\left(\Phi_1(x), \dots, \Phi_M(x)\right)~.
\end{align}
The essential restriction is that $\Phi$ is a permutation invariant mapping, typically of the form $\Phi_m(x) = \sum_{n=1}^N \phi_m(x_n)$ for some map $\phi_m: \mathbb{C} \rightarrow \mathbb{C}$.
In order to parameterize our student network as a DeepSets model, we will make the simplest possible choices, while preserving its non-linear essence. To define our student network, we consider the symmetric embedding $\Phi$ as a one-layer neural network with no bias terms:
\begin{align}
\Phi_m(x) &= \sum_{n=1}^N \sigma(a_{m} x_n)~, 
\end{align}
for i.i.d. complex weights sampled uniformly from the complex circle $a_m \sim \S$ and some activation $\sigma: \mathbb{C} \to \mathbb{C}$.  And given some link function $g: \mathbb{C} \to \mathbb{C}$, we'll consider the mapping $\rho$ as:
\begin{align}
    \rho_w(\cdot) = g(\langle w, \cdot \rangle)~,
\end{align}
where $w \in \mathbb{C}^M$ are our trainable weights. Putting all together, our student network thus becomes 
\begin{align}
F_w  : \mathcal{X} &\to \mathbb{C} \nonumber \\
x & \mapsto F_w(x) := g(\langle w, \Phi(x) \rangle)~.   
\end{align}
In other words, $F_w$ corresponds to a DeepSets network where the first and third layer weights are frozen, and only the second layer weights (with no biases) are trained.

The first fact we need is that, through simple algebra, the student may be rewritten in the form of a single-index model.
\begin{proposition}\label{prop:A_matrix}
    There is matrix $A \in \mathbb{C}^{\infty \times M}$ depending only on the activation $\sigma$ and the frozen weights $\{a_m\}_{m=1}^M$ such that 
    \begin{align}
        g(\langle w, \Phi(x) \rangle) = g(\langle Aw, p(x)\rangle)~.
    \end{align}
\end{proposition}

\subsection{Hermite-like Identity}

In the vanilla single index setting, the key to giving an explicit expression for the expected loss (for Gaussian inputs) is a well-known identity of Hermite polynomials~\citep{o2021analysis,jacobsen1996laplace}.  If $h_k$ denotes the Hermite polynomial of degree $k$, this identity takes the form
\begin{align}
    \langle h_k(\langle \cdot, u\rangle), h_l(\langle \cdot, v\rangle) \rangle_{\gamma_n} = \delta_{kl} k! \langle u, v \rangle^k~,
\end{align}
where $u, v \in \mathbb{R}^n$ and $\gamma_n$ is the standard Gaussian distribution on $n$ dimensions.

In our setting, as it turns out, one can establish an analogous identity, by considering a different input probability measure, and a bound on the degree of the link function.  We will choose our input domain $\mathcal{X} = (\S)^N$, and the input distribution we will consider is the set of eigenvalues of a Haar-distributed unitary matrix in dimension $N$~\citep{diaconis1994eigenvalues}, or equivalently the squared Vandermonde density over $N$ copies of the complex unit circle~\citep{macdonald1998symmetric}.
We'll interchangeably use the notation $\mathbb{E}_{x \sim V}[f(x) \conj{g(x)}] = \langle f, g \rangle_V$.

\begin{proposition}\label{prop:hermite}
    Consider $h,\tilde{h} \in \mathbb{C}^\infty$ with bounded $L_2$ norm.  For exponents $k, l$ with $k \leq \sqrt{N}$, if $h$ is only supported on the first $\sqrt{N}$ elements, then:
    \begin{align}
    \langle \langle h, p \rangle^k, \langle \tilde{h}, p\rangle^l \rangle_\V &= \delta_{kl} k! \langle h, \tilde{h} \rangle^k~.
\end{align}
\end{proposition}
The crucial feature of this identity is that the assumptions on support and bounded degree only apply to $\langle h, p \rangle^k$, with no restrictions on the other term.  In our learning problem, we can use this property to make these assumptions on the teacher function, while requiring no bounds on the terms of the student DeepSets architecture.

In order to take advantage of the assumptions on the support of $h$ and the degree in the above proposition, we need to make the following assumptions on our teacher link function $f$ and our true direction $h^*$:

\begin{assumption}\label{ass:teacher_link}
    The link function $f$ is analytic and only supported on the first $\sqrt{N}$ degree monomials, i.e.
    \begin{align}
        f(z) = \sum_{j=1}^{\sqrt{N}} \frac{\alpha_j}{\sqrt{j!}} z^j
    \end{align}
    Furthermore, the vector $h^*$ is only supported on the first $\sqrt{N}$ elements.
\end{assumption}
Although this assumption is required to apply the orthogonality property for our loss function in the following sections, we note that in principle, including exponentially small terms of higher degree in $f$ or higher index in $h^*$ should have negligible effect. Moreover, one should interpret this assumption as silently disappearing in the high-dimensional regime $N \to \infty$. For simplicity, we keep this assumption to make cleaner calculations and leave the issue of these small perturbations to future work.

\subsection{Information Exponent}
Because Proposition~\ref{prop:hermite} takes inner products of monomials, it alludes to a very simple characterization of information exponent.  Namely:
\begin{definition}
    Consider an analytic function $f: \mathbb{C} \rightarrow \mathbb{C}$ that can be written in the form
    \begin{align}
        f(z) = \sum_{j=0}^\infty \frac{\alpha_j}{\sqrt{j!}} z^j 
    \end{align}
    Then the \emph{information exponent} is defined as $s = \inf\{ j \geq 1: \alpha_j \neq 0\}$.
\end{definition}
Similar to the Gaussian case~\citep{arous2021online,biettilearning2022}, the information exponent $s$ will control the efficiency of learning.  Assuming $|\alpha_s|$ is some non-negligible constant, the value of $s$ will be far more important in governing the convergence rate. 

\subsection{Choosing a learnable loss}

There are two subtleties to choosing an appropriate loss function.  Namely, the necessity of a correlational loss (with regularization), and the necessity of choosing the student and teacher link functions to be distinct.

At first glance, it is tempting to simply define a loss of the form
\begin{align}
    \tilde{L}(w) = \mathbb{E}_{x \sim V} |F(x) - F_w(x)|^2 = \mathbb{E}_{x \sim V} \left[\left|f(\langle h^*, p(x)\rangle) - f(\langle Aw, p(x)\rangle)\right|^2 \right]~.
\end{align}
However, the Deepsets student model is not degree limited, that is the support of $Aw$ is not restricted to the first $\sqrt{N}$ terms of the powersum expansion. In other words, expanding this loss will require calculating the term $\|f(\langle Aw, p\rangle)\|_V^2$, which will contain high degree terms that cannot be controlled with Proposition~\ref{prop:hermite}.  One could avoid this issue by choosing the activation such that $Aw$ only contains low-index terms, but we want to consider larger classes of activations and enforce fewer restrictions.

One can instead consider a correlational loss. 
In this case, in order to make the objective have a bounded global minimum, it's necessary to either regularize $w$, or project at every step of SGD, which is the strategy taken in~\citet{damian2023smoothing}.
In our setting, this projection would correspond to projecting $w$ to the ellipsoid surface $\|Aw\| = 1$.  This projection would require solving an optimization problem at every timestep~\citep{pope2008algorithms}.  To avoid this impracticality, we instead consider regularization.

Then with complete knowledge of the link function $f$, specifically its monomial coefficients, we can now define the correlational loss
\begin{align}
    \hat{L}(w) &= \mathbb{E}_{x \sim V}\left[-\Re{f(\langle h^*, p(x) \rangle\overline{f(\langle Aw, p(x) \rangle)}}\right] + \sum_{i=j}^{\sqrt{N}} \frac{|\alpha_j|^2}{2} \|Aw\|^{2j}~.
\end{align}
This loss enjoys benign optimization properties, as shown by the following proposition:
\begin{proposition}\label{prop:silly_loss}
    If there exist coprimes $k, l$ with $\alpha_k, \alpha_l \neq 0$, and $h^*$ is in the range of $A$, then $\hat{L}$ exclusively has global minima at all $w$ such that $Aw = h^*$.
\end{proposition}
However, unlike the real case, complex weights causes issues for learning this objective.  Namely, this objective can be written as a non-convex polynomial in $\cosine \theta$ where $\theta$ is the angle of $\langle Aw, h^*\rangle$ in polar coordinates.

Therefore, we consider a different choice of student link function that will enable a simpler analysis of the dynamics.  For the choice of $g(z) = \frac{\alpha_s}{|\alpha_s|\sqrt{s!}} z^s$, we instead consider the loss:
\begin{align}
    L(w) &= \mathbb{E}_{x \sim V}\left[-\Re{f(\langle h^*, p(x) \rangle\overline{g(\langle Aw, p(x) \rangle)}}\right] + \frac{|\alpha_s|}{2} \|Aw\|^{2s} \\
    &= -|\alpha_s|\Re{\langle Aw, h^* \rangle^s} + \frac{|\alpha_s|}{2} \|Aw\|^{2s}~.
\end{align}
We note that~\citet{dudeja2018learning} used a similar trick of a correlational loss containing a single orthogonal polynomial in order to simplify the learning landscape.  The global minima of this loss, and in fact the dynamics of gradient flow on it, will be explored in the sequel.

\section{Related Work}

\subsection{Single Index Learning}

The conditions under which single-index model learning is possible have been well-explored in previous literature.  The main assumptions that enable provably learning under gradient flow / gradient descent are monotonicity of the link function~\citep{kakade2011efficient,kalai2009isotron,shalev2010learning,yehudai2020learning} or Gaussian input distribution~\citep{arous2021online}.  The former assumptions essentially corresponds to the setting where the information exponent $s = 1$, as it will have positive correlation with a linear term.  Under the latter assumption, the optimal sample complexity was achieved in~\citet{damian2023smoothing}, with study of learning when the link function is not known in~\citet{biettilearning2022}.

When both assumptions are broken, the conditions on the input distribution of rotation invariance or approximate Gaussianity are nevertheless sufficient for learning guarantees~\citep{bruna2023single}.  But more unusual distributions, especially in the complex domain that is most convenient for symmetric networks, are not well studied.

\subsection{Symmetric Neural Networks}

The primary model for symmetric neural networks was introduced in~\citet{zaheer2017deep} as the DeepSets model.  There are many similar models that enforce permutation invariance~\citep{qi2017pointnet,santoro2017simple,lee2019set}, though we focus on DeepSets because of its relationship with the power sum polynomials and orthogonality~\citep{zweig2022exponential}.  We are not aware of any other works that demonstrates provable learning of symmetric functions under gradient-based methods.

\section{Provably Efficient Recovery with Gradient Flow}

\subsection{Defining the Dynamics}

The gradient methods considered in~\citet{arous2021online, arous2022high} are analyzed by reducing to a dimension-free dynamical system of the so-called summary statistics. For instance, in the vanilla single-index model, the summary statistics reduce to the scalar correlation between the learned weight and the true weight.  
In our case, we have three variables, owing to the fact that the correlation is complex and represented by two scalars, and a third variable controlling the norm of the weight since we aren't using projection.

Note that although our weight vector $w$ is complex, we still apply regular gradient flow to the pair of weight vectors $w_R, w_C$ where $w = w_R + iw_C$.  Furthermore, we use the notation $\nabla := \nabla_w = \nabla_{w_R} + i \nabla_{w_C}$.  With that in mind, we can summarize the dynamics of our gradient flow in the following Theorem. 
\begin{theorem}\label{thm:dynamic}
    Given a parameter $w$, consider the summary statistics $m = \langle Aw, h^*\rangle \in \mathbb{C}$ and $v = \|P_{h^*}^\perp A w\|^2$ where $P_{h^*}^\perp$ is projection onto the orthogonal complement of $h^*$.  Let the polar decomposition of $m$ be $re^{i\theta}$.

    Then given the preconditioned gradient flow given by
    \begin{align}
        \dot{w} = - \frac{1}{s|\alpha_s|} (A^\dag A)^{-1} \nabla L(w)~,
    \end{align}
    the summary statistics obey the following system of ordinary differential equations:
\begin{align}
    \dot{r} &= (1-\delta) r^{s-1} \cos s \theta -(v+r^2)^{s-1} r~,\\
    \frac{d}{dt} \cos s \theta &= (1-\delta) sr^{s-2} (1 - \cos^2 s\theta)~,\\
    \dot{v} &= 2 \delta r^s \cos s\theta - 2 (v + r^2)^{s-1}v~,
\end{align}
where $\delta := 1 - \|P_A h^*\|^2$ and $P_A$ is the projection onto the range of $A$.
\end{theorem}
The proof is in Appendix \ref{proof:dynamic}.  The main technical details come from using Wirtinger calculus to determine how the real and imaginary parts of $w$ evolve under the flow.  Additionally, the correct preconditioner (intuitive from the linear transform of $w$) is crucial for reducing the dynamics to only three summary statistics, and converting to dynamcis on $\cos s \theta$ rather than $\theta$ itself simplifies the description of the learning in the next section dramatically.

\subsection{Provable Learning}

These dynamics naturally motivate the question of learning efficiency, measured in convergence rates in time in the case of gradient flow.  Our main result is that, under some assumptions on the initialization of the frozen weights $\{a_m\}_{m=1}^M$ and the initialized weight vector $w_0$, the efficiency is controlled by the initial correlation with the true direction and the information exponent, just as in the Gaussian case.

\begin{theorem}\label{thm:learning}
    Consider a fixed $\epsilon > 0$.  Suppose the initialization of $w_0$ and $(a_m)_{m=1}^M$ are such that:
    \begin{enumerate}[label=(\roman*)]
        \item Small correlation and anti-concentration at initialization: $0 < r_0 \leq 1$,
        \item Initial phase condition: $\cos s \theta_0 \geq 1/2$, 
        \item Initial magnitude condition for $Aw$: $v_0 = 1 - r_0^2$,
        \item Small Approximation of optimal error: $\delta \leq \min(\epsilon/2, O(s^{-s}r_0^4))$.
    \end{enumerate}
    
    Then if we run the gradient flow given in Theorem~\ref{thm:dynamic} we have $\epsilon$ accuracy in the sense that:
    \begin{align}
            r_T &\geq 1 - \epsilon~,~ \cos s \theta_T \geq 1 - \epsilon ~,~ v_T \leq \epsilon
    \end{align}
    after time $T$, where depending on the information exponent $s$:
    \begin{align}
        T \leq \begin{cases} 
      O\left(\log \frac{1}{\epsilon} \right) & s = 1 ~,\\
      O\left(2^{s^2} r_0^{-4s} + \log \frac{1}{\epsilon} \right) & s > 1~.
   \end{cases}
    \end{align}
\end{theorem}

\begin{remark}
    We note that we only recover $\cos s \theta \approx 1$, rather than a guarantee that $\theta \approx 0$, and so the hidden direction is only determined up to scaling by a $s$th root of unity.  This limitation is may appear to be an issue with the choice of the student link function $g$, but it is unavoidable: if the teacher link function $f(z) = \frac{1}{\sqrt{s!}} z^s$, one can calculate that for any choice of $g$, $L(w)$ is invariant to scaling $w$ by an $s$th root of unity.
\end{remark}

\subsection{Initialization Guarantees}

In order to apply the gradient flow bound proved in Theorem~\ref{thm:learning}, it only remains to understand when the assumptions on initialization are met.  Unlike the single-index setting with Gaussian inputs, the initial correlation is not guaranteed to be on the scale of $\frac{1}{\sqrt{N}}$, but will depend on the activation function and the random weights in the first layer.  Let us introduce the assumptions we'll need:

\begin{assumption}\label{ass:act}
    We assume an analytic activation $\sigma(z) = \sum_{k=0}^\infty c_k z^k$, with the notation $\sigma_+ := \max_{1 \leq k \leq N} |c_k|\sqrt{k}$ and $\sigma_- := \min_{1 \leq k \leq \sqrt{N}} |c_k|\sqrt{k}$.
    We further assume:
    \begin{enumerate}[label=(\roman*)]
        \item $c_k = 0$ iff $k = 0$,
        \item $\sigma$ analytic on the unit disk,
        \item $1/\sigma_- = O(\mathrm{poly}(N))$,
        \item $\sum_{k=N+1}^\infty k|c_k|^2  \leq e^{-O(\sqrt{N})}$.
    \end{enumerate}

\end{assumption}

The first two conditions are simply required for the application of Proposition~\ref{prop:hermite}, as the powersum vector $p$ is built out of polynomials induced by the activation and does not include a constant term.  The second two conditions concern the decay of the coefficients of $\sigma$, in the sense that the decay must start slow but eventually become very rapid.  These conditions are necessary mainly for ensuring the Small Approximation of optimal error condition:

\begin{lemma}\label{lem:projection}
    Let $\sigma$ satisfy Assumption~\ref{ass:act}, and assume $M = O(N^3)$.  Then for any unit norm $h^* \in \mathbb{C}^\infty$ that is only supported on the first $\sqrt{N}$ elements, with probability $1 - 2\exp(-O(N))$:
    $$1 - \|P_A h^* \|^2 \leq e^{-O(\sqrt{N})}~.$$
\end{lemma}

Lastly, we can choose an initialization scheme for $w$ which handily ensures the remaining assumptions we need to apply Theorem~\ref{thm:learning}.  The crucial features of $\sigma$ are similar to the previous result.  Namely, we want the initial correlation $r_0$ to be non-negligible because this directly controls the runtime of gradient flow.  Slow initial decay with fast late decay of the $\sigma$ coefficients directly implies that $Aw_0$ has a lot of mass in the first $\sqrt{N}$ indices and very little mass past the first $N$ indices.  These requirements rule out, say, $\exp$ as an analytic activation because the coefficients decay too rapidly.

\begin{lemma}\label{lem:init_concentration}
    Suppose $w$ is sampled from a standard complex Gaussian on $M$ variables.  It follows that if we set $w_0 = \frac{w}{\|A w\|}$, and use the summary statistics from Theorem~\ref{thm:dynamic}, then with probability $1/3 - 2\exp(-O(N))$ and any $h^*$ as in Lemma~\ref{lem:projection}
    \begin{enumerate}[label=(\roman*)]
        \item $1 \geq r_0 \geq c\frac{\sigma_-}{\sigma_+\sqrt{M}}$ for some universal constant $c > 0$,
        \item $\cos s \theta_0 \geq 1/2$, 
        \item $v_0 = 1 - r_0^2$.
    \end{enumerate}
\end{lemma}



Finally, we consider a straightforward choice of $\sigma$ that meets Assumption~\ref{ass:act} so that we can arrive at an explicit complexity bound on learning:
\begin{corollary}[Non-asymptotic Rates for Gradient Flow]\label{coro:learning}
    And Consider $\xi = 1 - \frac{1}{\sqrt{N}}$ and the specific choice of activation $$\sigma(z) = \arctan{\xi z} + \xi z \arctan{\xi z}~. $$

    Suppose we initialize $w$ from a standard complex Gaussian in dimension $M$ with $M = O(N^3)$, and $\{a_m\}_{m=1}^M \sim \S$ iid.  Furthermore, treat $s$ and $\epsilon$ as constants relative to $N$.  Then with probability $1/3 - 2\exp(-O(N))$, we will recover $\epsilon$ accuracy in time
    \begin{align}
        T \leq \begin{cases} 
      O\left(\log \frac{1}{\epsilon} \right) & s = 1 \\
      O\left(2^{s^2} N^{7s} + \log \frac{1}{\epsilon} \right) & s > 1~.
   \end{cases}
    \end{align}
\end{corollary}
\begin{proof}
    By Proposition~\ref{prop:arctan}, the activation $\sigma$ given in the corollary statement satisfies Assumption~\ref{ass:act}, so we can apply Lemma~\ref{lem:init_concentration} and Lemma~\ref{lem:projection} to satisfy the requirements of Theorem~\ref{thm:learning}.  In particular, the fourth condition is given by assuming $e^{-O(\sqrt{N})} \leq \min(\epsilon/2, O(s^{-s}r_0^4))$ which is true when $s$ is constant, and $\epsilon$ and $r_0$ are at most polynomial compared to $N$.

    Note that $\sigma_+ = O(1)$ and $\sigma_- = O\left(\frac{1}{N^{1/4}}\right)$, so it follows that $r_0 \geq O\left(\frac{1}{N^{7/4}} \right)$ with probability $1/3 - 2\exp(-O(N))$.  Conditioning on this bound gives the desired bound on the time for $\epsilon$ accuracy.
\end{proof}

Hence, we have a rate that, for $s = O(1)$, is not cursed by dimensionality to recover the true hidden direction $h^*$.  As mentioned above, there are two caveats to this recovery: $w$ is only recovered up to an $s$th root of unity, and to directly make predictions of the teacher model would require using the teacher link function rather than using the student model directly.

Since this result concerns gradient flow over the population loss, a natural question is what barriers exist that stymie the SGD analysis of recent single index papers~\citep{arous2021online,damian2023smoothing,bruna2023single}. These works treat the convergence of SGD by a standard drift and martingale argument, where the drift follows the population gradient flow, and the martingales are shown to be controlled via standard concentration inequalities and careful arguments around stopping times.
Applying these tactics to a discretized version of the dynamics given in Theorem~\ref{thm:dynamic} mainly runs into an issue during the first phase of training.  Unlike in~\citet{arous2021online} where the drift dynamics have the correlation monotonically increasing towards $1$, at the start of our dynamics the correlation magnitude $r$ and the ``orthogonal" part of the learned parameter $v$ are both decreasing (with high probability over the initialization).  Showing that this behavior doesn't draw the model towards the saddle point where $r=0$ requires showing that $v$ decreases meaningfully faster than $r$, i.e. showing that $\frac{d}{dt} \log \frac{r^2}{v}$ is positive.  It's not clear what quality of bounds the martingale concentration inequalities would provide for this quantity, and we leave for future work if the six stage proof of the dynamics behavior could be successfully discretized.

\section{Experiments}


To study an experimental setup for our setting, we consider the student-teacher setup outlined above with gradient descent.  We consider $N = 25$, $M = 100$, and approximate the matrix $A$ by capping the infinite number of rows at $150$, which was sufficient for $1 - \|P_A h^*\|^2 \leq 0.001$ in numerical experiments.  For the link function $f$, we choose its only non-zero monomial coefficients to be $\alpha_3 = \alpha_4 = \alpha_5 = \frac{1}{\sqrt{3}}$.  And correspondingly, $g$ simply has $\alpha_3 = 1$ and all other coefficients at zero.

We choose for convenience an activation function such that $A_{km} = \left(\frac{N-1}{N}\right)^ka_m^k$.  We make this choice because, while obeying all the assumptions required in Assumption~\ref{ass:act}, this choice implies that the action of $A$ on the elementary basis vectors $e_j$ for $1 \leq j \leq \sqrt{N}$ is approximately distributed the same.  This choice means that $\|P_A h^*\|$ is less dependent on the choice of $h^*$, and therefore reduces the variance in our experiments when we choose $h^*$ uniformly among unit norm vectors with support on the first $\sqrt{N}$ elements, i.e. uniformly from the complex sphere in degree $\sqrt{N}$.

Under this setup, we train full gradient descent on 50000 samples from the Vandermonde $V$ distribution under 20000 iterations.  The only parameter to be tuned is the learning rate, and we observe over the small grid of $[0.001, 0.0025, 0.005]$ that a learning rate of $0.0025$ performs best for the both models in terms of probability of $r$ reaching approximately $1$, i.e. strong recovery.

As described in Theorem~\ref{thm:dynamic}, we use preconditioned gradient descent using $(A^\dag A)^{-1}$ as the preconditioner, which can be calculated once at the beginning of the algorithm and is an easy alteration to vanilla gradient descent to implement.  We use the pseudoinverse for improved stability in calculating this matrix, although we note that this preconditioner doesn't introduce stability issues into the updates of our summary statistics, even in the case of gradient descent.  Indeed, even if one considers the loss $L(w)$ under an empirical expectation rather than full expectation, the gradient $\nabla L(w)$ can still be seen to be written in the form $A\dag v$ for some vector $v$.  If one preconditions this gradient by $(A^\dag A)^{-1}$, and observes that the summary statisics $m$ and $v$ both depend on $Aw$ rather than $w$ directly, it follows that the gradient update on these statistics is always of the form $A(A^\dag A)^{-1} A^\dag = P_A$, so even in the empirical case this preconditioner doesn't introduce exploding gradients.


\begin{figure*}[ht]
\centering

\begin{subfigure}{.45\textwidth}
  \centering
  \includegraphics[width=1.\linewidth]{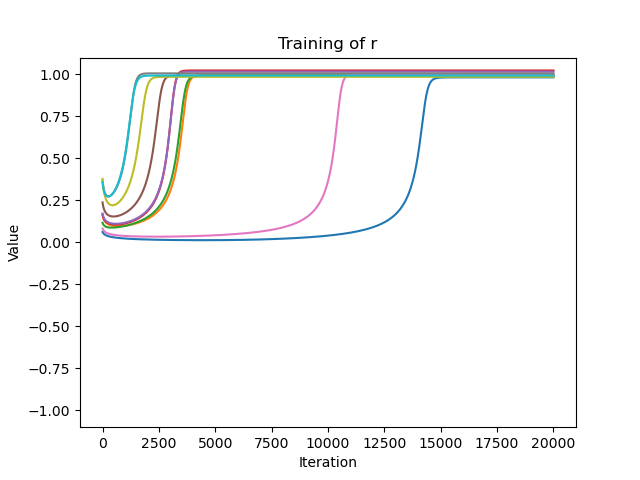}
\end{subfigure}%
\begin{subfigure}{.45\textwidth}
  \centering
  \includegraphics[width=1.\linewidth]{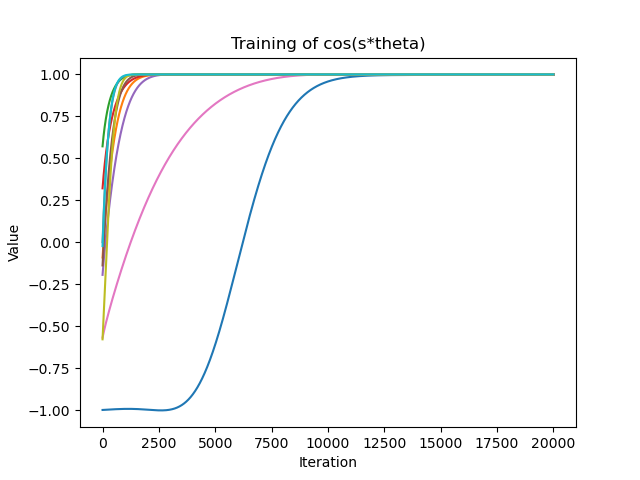}
\end{subfigure}%

\begin{subfigure}{.45\textwidth}
  \centering
  \includegraphics[width=1.\linewidth]{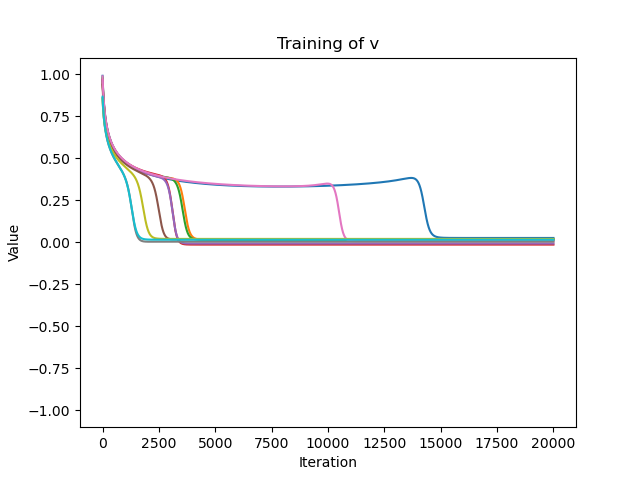}
\end{subfigure}%
\begin{subfigure}{.45\textwidth}
  \centering
  \includegraphics[width=1.\linewidth]{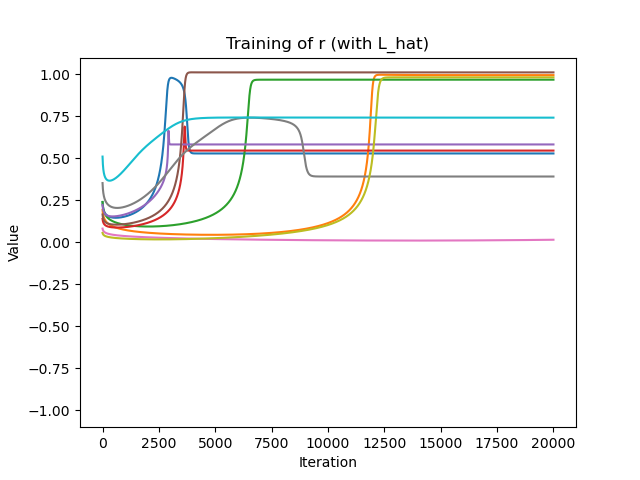}
\end{subfigure}%

\caption{The learning trajectory, over ten idependent runs, of the three summary statistics in the case of our chosen loss function $L$, and the trajectory of the $r$ statistic for the more complicated loss function $\hat{L}$}
\vspace{-0.5cm}
\label{fig:plots}
\end{figure*}

\section{Discussion}

\subsection{Experimental Results}

The outcomes of our experiments are given in Figure~\ref{fig:plots}.  We observe very high rates of strong recovery using the loss $L$.  For the loss $\hat{L}$, we note that $r$ often becomes stuck, indicating the model has reached a local minima.

We note that our analysis is somewhat pessimistic, as the experimental gradient descent on $L(w)$ will often achieve near perfect accuracy even if $\cos s \theta_0 < 0$.  This is mainly an issue of proof technique: although $\cos s \theta$ is always increasing under the dynamics, $r$ is necessarily decreasing for as long as $\cos s \theta$ is negative.  It is quite subtle to control whether $\cos s \theta$ will become positive before $r$ becomes extremely small, and the initialization of $r$ is the main feature that controls the runtime of the model.  However the empirical results suggest that a chance of success $> 1/2$ is possible under a more delicate analysis.

However, the analysis given in the proof of Theorem~\ref{thm:learning} does accurately capture the brief dip in the value of $r$ in the initial part of training, when the regularization contributes more to the gradient than the correlation until $\cos s \theta$ becomes positive.

Because we can only run experiments on gradient descent rather than gradient flow, we observe the phenomenon of search vs descent studied in~\citet{arous2021online}, where the increase in the corrleation term $r$ is very slow and then abruptly increases.
For the model trained with $\hat{L}$, we observe that there is much greater likelihood of failure in the recovery, as $r$ appears to become stuck below the optimal value of $1$.

\subsection{Extensions}

The success of this method of analysis depends heavily on the Hermite-like identity in Proposition~\ref{prop:hermite}.  In general, many of the existing results analyzing single index models need to assume either Gaussian inputs, or uniformly distributed inputs on the Boolean hypercube (see for example~\citet{abbe2023sgd}).  In some sense, this works cements the inclusion of the Vandermonde distribution in this set of measures that enable clean analysis.  The proof techniques for these three measures are quite disparate, so it remains open to determine if there is a wider class of  ``nice" distributions where gradient dynamics can be succcessfully analyzed.

Additionally, the success of the multi-layer training in~\citet{biettilearning2022,mahankali2023beyond} suggests that simultaneously training the frozen first layer weights may not prohibit the convergence analysis.  The matrix $A$ depends on the first layer weights through a Vandermonde matrix (see $X$ in the proof of Lemma~\ref{lem:projection}), and the simple characterization of the derivative of a Vandermonde matrix alludes to further possibilities for clean analysis.

\subsection{Limitations}

A first limitation is the focus of this work on complex inputs, analytic activations, and fixed input distribution (namely the squared Vandermonde density).  Although complex analytic functions are less commonly studied in the literature, they do still appear in settings like quantum chemistry~\citep{beau2021parent,langmann2005method}.  Regarding the focus on the Vandermonde distribution, we note this is similiar to the vanilla single-index setting in the restriction to Gaussian inputs, under which the theory is particularly powerful, simplest and understanding of non-Gaussian data is still nascent.

A second limitation is that this work focuses on input distributions over sets of scalars, whereas typically symmetric neural networks are applied to sets of high-dimensional vectors.  Proposition~\ref{prop:hermite} does not work out of the box for these settings without a high-dimensional analogue of the inner product $\langle \cdot, \cdot \rangle_V$ with similar orthogonality properties.  It is possible to define such an inner products on the so-called multisymmetric powersums with similar orthogonality~\citep{zweig2022exponential}, and we leave to future work the question of whether such inner products could grant similar guarantees about the learning dynamics in this more realistic setting.

\section{Conclusion}

In this work we've shown a first positive result that quantifies the ability of gradient descent to perform symmetric feature learning, by  adapting and extending the tools of two-layer single index models. In essence, this is made possible by a `miracle', namely the fact that certain powersum expansions under the Vandermonde measure enjoy the same semigroup structure as Hermite polynomials under the Gaussian measure (Proposition \ref{prop:hermite}) --- leading to a dimension-free summary statistic representation of the loss. Although the resulting dynamics are more intricate than in the Euclidean setting, we are nonetheless able to establish quantitative convergence rates to `escape the mediocrity' of initialization, recovering the same main ingredients as in previous works \cite{arous2021online, abbe2022merged}, driven by the information exponent. To our knowledge, this is the first work to show how learning with gradient based methods necessarily succeeds in this fully non-linear 
 (i.e. not in the NTK regime) setting. Nevertheless, there are many lingering questions.

As discussed, one limitation of the analysis is the reliance on gradient flow rather than gradient descent.  We hope that in future work we'll be able to effectively discretize the dynamics, made more challenging by the fact that one must track three parameters rather than simply the correlation. 
Still, we observe theoretically and empirically that the symmetric single index setting demands a number of unusual choices, such as a correlation loss and distinct student and teacher link function, in order to enable efficient learning.  And in a broader scheme, if one remembers the perspective of DeepSets as a very limited form of a three-layer architecture, the issue of provable learning for deeper, more realistic architectures stands as a very important and unexplored research direction --- and where Transformers with planted low-dimensional structures appear as the next natural question.


\bibliography{iclr2024_conference}
\bibliographystyle{iclr2023_conference}

\newpage
\appendix

\section{Proof of Proposition~\ref{prop:A_matrix}}

The proposition is true for any activation with a Laurent series, but we will only prove it for activations satisfying Assumption~\ref{ass:act} since that's the only setting we'll require it.

Consider an analytic activation $\sigma$ with no constant term, given as

\begin{align}
    \sigma(z) = \sum_{k=1}^\infty c_k z^k
\end{align}

And remind the deepsets features map

\begin{align}
\Phi_m(x) &= \sum_{n=1}^N \sigma(a_{m} x_n) \\
\end{align}

where we have neurons without bias terms.

Then for a weight $w \in \mathbb{C}^M$, one can quickly see

\begin{align}
    \langle w, \Phi(x) \rangle & = \sum_{m=1}^M w_m \sum_{n=1}^N \sigma(a_m x_n)  \\
    & = \sum_{m=1}^M w_m \sum_{n=1}^N \sum_{k=1}^\infty c_k (a_m x_n)^k \\
    & = \sum_{k=1}^\infty \sum_{m=1}^M w_m c_k a_m^k \sqrt{k} p_k(x) \\
    & = \langle Aw, p(x) \rangle
\end{align}

where $A_{km} = c_k \sqrt{k} a_m^k $

\section{Proof of Proposition~\ref{prop:hermite}}

We require some definitions to use the machinery of symmetric polynomials.

\begin{definition}
    An \emph{integer partition} $\lambda$ is non-increasing, finite sequence of positive integers $\lambda_1 \geq \lambda_2 \geq \dots \geq \lambda_k$.  The weight of the partition is given by $|\lambda| = \sum_{i=1}^k \lambda_i$.  The length of a partition $l(\lambda)$ is the number of terms in the sequence.
\end{definition}

Then we characterize a product of powersums by:
\begin{align}
    p_\lambda(x) = \prod_i p_{\lambda_i}(x)
\end{align}

Finally, define the combinatorial constant $t_\lambda = \prod_{i=1}^{|\lambda|} (m_i)!$ where $m_i$ denotes the number of parts of $\lambda$ equal to $i$.

\begin{theorem}[{\cite[Chapter VI (9.10)]{macdonald1998symmetric}} ]\label{thm:hall-inner-product}
    For partitions $\lambda, \mu$ with $|\lambda| \leq N$:
    \begin{align}
        \langle p_\lambda, p_\mu \rangle_\V = t_\lambda \ind_{\lambda = \mu}
    \end{align}
\end{theorem}

With that in mind, let's consider the inner product of two simple single-index functions.

Let $p = [p_1, p_2, \dots]$ be an infinite vector of powersums, and choose exponents $i, j$ with $i \leq \sqrt{N}$.  Then for any $h, \tilde{h} \in \mathbb{C}^{\infty}$ such that $h$ is only supported on the first $\sqrt{N}$ entries:

\begin{align}
    \langle \langle h, p \rangle^i, \langle \tilde{h}, p\rangle^j \rangle_\V &= \left\langle \sum_{|\alpha| = i} \binom{i}{\alpha}h^\alpha \overline{p^\alpha}, \sum_{|\alpha| = j} \binom{j}{\alpha} \tilde{h}^\alpha \overline{p^\alpha}, \right\rangle_\V \\
    & = \delta_{ij} \sum_{|\alpha| = i} \binom{i}{\alpha}^2 \langle p^\alpha, p^\alpha \rangle_\V h^\alpha\overline{\tilde{h}^\alpha}\\
    &= \delta_{ij} \sum_{|\alpha| = i} \binom{i}{\alpha}^2 \left(\prod_{k=1}^{\sqrt{N}} \alpha_k ! \right) h^\alpha\overline{\tilde{h}^\alpha} \\
    &= \delta_{ij} i! \sum_{|\alpha| = i} \binom{i}{\alpha} h^\alpha\overline{\tilde{h}^\alpha} \\
    &= \delta_{ij} i! \langle h, \tilde{h} \rangle^i
\end{align}

\section{Proof of Propsition~\ref{prop:silly_loss}}

Applying Proposition~\ref{prop:hermite} and using assumptions on the degree bound on $f$ and the support of $h^*$, we can write:

\begin{align}
    \left\langle f(\langle h^*, p\rangle), f(\langle Aw, p \rangle) \right\rangle_V &= \left\langle \sum_{j=1}^{\sqrt{N}} \frac{\alpha_j}{\sqrt{j!}} \langle h^*, p\rangle^j, \sum_{j=1}^{\sqrt{N}}\frac{\alpha_j}{\sqrt{j!}} \langle Aw, p \rangle^j \right\rangle_V\\
    &= \sum_{j=1}^{\sqrt{N}} |\alpha_j|^2 \langle h^*, Aw \rangle^j
\end{align}

Hence we have

\begin{align}
    \hat{L}(w) &= E_{x \sim V}\left[-\Re{f(\langle h^*, p(x) \rangle\overline{f(\langle Aw, p(x) \rangle)}}\right] + \sum_{j=1}^{\sqrt{N}} \frac{|\alpha_j|^2}{2} \|Aw\|^{2j} \\
    &= -\frac{1}{2}\left\langle f(\langle h^*, p\rangle), f(\langle Aw, p \rangle) \right\rangle_V -\frac{1}{2}\overline{\left\langle f(\langle h^*, p\rangle), f(\langle Aw, p \rangle) \right\rangle_V} + \sum_{j=1}^{\sqrt{N}} \frac{|\alpha_j|^2}{2} \|Aw\|^{2j}\\
    &= -\sum_{j=1}^{\sqrt{N}} |\alpha_j|^2 \Re{\langle h^*, Aw \rangle^j} + \frac{|\alpha_j|^2}{2} \|Aw\|^{2j}
\end{align}

Now, we use the same notation as in Theorem~\ref{thm:dynamic} and introduce variables $m = \langle Aw, h^*\rangle = re^{i\theta}$ and $v = \|Aw\|^2 - r^2$, such that we can write:

\begin{align}
    \hat{L}(w) = \sum_{j=1}^{\sqrt{N}} |\alpha_j|^2 \left(-r^j \cos j \theta + \frac{1}{2} (v +r^2)^{j}\right)
\end{align}

Because $r \geq 0$ and $v \geq 0$, this loss can be minimized by setting $v = 0$ and any $\theta$ where $\cos j \theta = 1$ for all $j$ with $\alpha_j \neq 0$.  Since we assume there are distinct indices $i, j$ that are coprime with non-zero support, we require $i \theta$ and $j \theta$ to both be multiples of $2\pi$, which is only possible if $\theta \equiv 0 \mod 2 \pi$.  Therefore:

\begin{align}
    \hat{L}(w) &= \sum_{j=1}^{\sqrt{N}} |\alpha_j|^2 \left(-r^j + \frac{1}{2} r^{2j}\right) \\
    &= C +\sum_{j=1}^{\sqrt{N}} \frac{|\alpha_j|^2}{2} \left(r^j - 1\right)^2
\end{align}

for some constant $C$, and this is minimized at $r = 1$.  Hence, if $r = 1, \theta \equiv 0 \mod 2\pi, v = 0$, it follows that $Aw = h^*$.

\section{Proof of Theorem~\ref{thm:dynamic}}\label{proof:dynamic}

Given the matrix $A$ and weight $w$, an identical calculation to the one in Proposition~\ref{prop:silly_loss} shows that
\begin{align}
    L(w) &= E_{x \sim V}\left[-\Re{f(\langle h^*, p(x) \rangle\overline{g(\langle Aw, p(x) \rangle)}}\right] + \frac{|\alpha_s|}{2} \|Aw\|^{2s} \\
    &= -|\alpha_s|\Re{\langle Aw, h^* \rangle^s} + \frac{|\alpha_s|}{2} \|Aw\|^{2s}
\end{align}

To calculate the gradient with respect to the real and imaginary parts of $w$, we use tools from Wirtinger calculus \citep{fischer2005precoding}.  Using the notation that $\nabla_{\overline{w}} = \frac{1}{2}(\nabla_{w_R} + i \nabla_{w_C})$ and the appropriate generalization of the chain rule, we have:

\begin{align}
    2\nabla_{\overline{w}} \Re{\langle Aw, h^* \rangle^s} &= \nabla_{\overline{w}} \left(\langle Aw, h^* \rangle^s + \overline{\langle Aw, h^* \rangle^s}\right) \\
    &= \nabla_{\overline{w}} \overline{\langle Aw, h^* \rangle}^s \\
    &= s \overline{\langle Aw, h^* \rangle}^{s-1} A^\dag h^*
\end{align}

Likewise,

\begin{align}
    2\nabla_{\overline{w}} \|Aw\|^{2s} &= 2s \|Aw\|^{2(s-1)} \nabla_{\overline{w}} \|Aw\|^2 \\
    &= 2s \|Aw\|^{2(s-1)} \nabla_{\overline{w}} \left(w^\dag A^\dag A w\right) \\
    &= 2s \|Aw\|^{2(s-1)} A^\dag A w
\end{align}

Thus, we have:

\begin{align}
    \nabla L &= \nabla_{w_R} L + i \nabla_{w_C} L \\
    &= 2\nabla_{\overline{w}} L \\
    &= -s|\alpha_s| \overline{\langle Aw, h^* \rangle}^{s-1}A^\dag h^* + s|\alpha_s| \|Aw\|^{2(s-1)} A^\dag Aw
\end{align}

We introduce the parameters

\begin{align}
    m &= \langle Aw, h^* \rangle = \langle w, A^\dag  h^*\rangle \\
    v &= \|P_{h^*}^\perp A w \|^2 = \|Aw\|^2 - |m|^2
\end{align}

And we consider preconditioned gradient flow of the form (where for complex variables we use similar notation that $\dot{w} = \dot{w_R} + \dot{w_C}i$):

\begin{align}
    \dot{w} &= -\frac{1}{s|\alpha_s|} (A^\dag A)^{-1} \nabla L \\
    &= \overline{m}^{s-1}(A^\dag A)^{-1}A^\dag h^* - \|Aw\|^{2(s-1)} w
\end{align}

It follows that

\begin{align}
    \dot{m} &= \langle \dot{w}, A^\dag  h^* \rangle\\
    &= \|P_A h^*\|^2 \overline{m}^{s-1} - (v + |m|^2)^{s-1} m
\end{align}

where $P_A = A(A^\dag A)^{-1}A^\dag$ is the orthogonal projection onto the range of $A$.

Let $m = a + bi = re^{i\theta}$, so we have $\dot{m} = \dot{a} + \dot{b}i$.  Thus

\begin{align}
    \dot{a} &= \|P_A h^*\|^2 r^{s-1} \cos (s-1)\theta - (v + r^2)^{s-1} r \cos \theta\\
    \dot{b} &= -\|P_A h^*\|^2 r^{s-1} \sin (s-1)\theta - (v + r^2)^{s-1} r \sin \theta
\end{align}

Now we do a change of variables, because $a = r \cos \theta$ and $b = r \sin \theta$, so 

\begin{align}
    \dot{a} &= \dot{r} \cos \theta - r \dot{\theta} \sin \theta \\
    \dot{b} &= \dot{r} \sin \theta + r \dot{\theta} \cos \theta \\
\end{align}

Rearranging, we can get the flow on $r$ and $\theta$:

\begin{align}
    \dot{r} &= \dot{a} \cos \theta + \dot{b} \sin \theta \\
    &= \|P_A h^* \|^2 r^{s-1} \cos s \theta -(v+r^2)^{s-1} r\\
    r \dot{\theta} &= -\dot{a} \sin \theta + \dot{b} \cos \theta \\
    &= - \|P_A h^* \|^2 r^{s-1} \sin s \theta\\
\end{align}

We can instead control the flow on $\cos s \theta$:

\begin{align}
    \frac{d}{dt} \cos s \theta = -\dot{\theta} s \sin s \theta = \|P_A h^*\|^2 sr^{s-2} \sin^2 s\theta
\end{align}

and calculate the flow on $v$:
\begin{align}
    \dot{v} &= 2 \Re{\langle A \dot{w}, Aw \rangle} - 2r\dot{r}\\
    &= 2 \left(r^s \cos s\theta - (v + r^2)^{s} - \|P_Ah^*\|^2 r^s \cos s \theta + (v+r^2)^{s-1}r^2 \right)\\
    &= 2 (1 - \|P_Ah^*\|^2) r^s \cos s\theta - 2 (v + r^2)^{s-1}v
\end{align}

Finally, introducing the notation $\delta = 1 - \|P_A h^*\|^2$, we have

\begin{align}
    \dot{r} &= (1-\delta) r^{s-1} \cos s \theta -(v+r^2)^{s-1} r\\
    \frac{d}{dt} \cos s \theta &= (1-\delta) sr^{s-2} (1 - \cos^2 s\theta)\\
    \dot{v} &= 2 \delta r^s \cos s\theta - 2 (v + r^2)^{s-1}v
\end{align}

\section{Proof of Theorem~\ref{thm:learning}}

We will use the following facts repeatedly in the below arguments.

First, because $\dot{r} \geq 0$ when $r = 0$, and $\dot{r} \leq 0 $ when $r = 1$, it follows that $r$ can never leave the range $[0,1]$.  Furthermore, note that $\cos s \theta$ is always non-decreasing.

\subsection{Case $s = 1$}

In the setting with information complexity equal to $1$, we immediately have the following identities:

\begin{align}
    \dot{r} &= (1-\delta)\cos \theta - r\\
    \frac{d}{dt} \cos \theta &\geq (1-\delta) (1 - \cos^2 \theta)\\
    \dot{v} &\leq 2 \delta - 2v
\end{align}

Let us address $v$ first.  From our assumptions, $\delta < \epsilon$, and so when $v \geq \epsilon$, $\dot{v}$ is negative.  It follows that a trajectory that begins below $\epsilon$ cannot ever exceed $\epsilon$.  In other words, if $v_0 \leq \epsilon$, $v$ can never exceed $\epsilon$ and we've achieved optimality. 

Otherwise, supposing $v_0 > \epsilon$, consider values of $t$ where $v_t > \delta$ so that the RHS of the inequality of $\dot{v}$ is strictly negative and we may write:

\begin{align}
    \frac{\dot{v}}{\delta - v} \geq 2
\end{align}

Integrating from $0$ to $t$ gives that

\begin{align}
    - \log |\delta - v_t| - (- \log |\delta - v_0|) \geq 2t
\end{align}

which yields the bound

\begin{align}
    v_t \leq \delta + (v_0 - \delta)e^{-2t} \leq \delta + e^{-2t}
\end{align}

By Lemma~\ref{lem:theta_dynamics},

\begin{align}
    \cos \theta_t \geq \tanh((1-\delta)t)
\end{align}

Finally, we consider $r$.


Choose $T_1 = \inf\{t \geq 0: v_t \leq \epsilon, \cos \theta_t \geq \frac{1-\epsilon/2}{1-\delta}\}$, and $T_2 = \inf\{t \geq T_1: r_t \geq 1-\epsilon\}$.  Note that one can easily confirm that $T_1 \leq O\left(\log \frac{1}{\epsilon} \right)$

Then for all $t \in [T_1, T_2)$, we have

\begin{align}
    \dot{r}_{t} = (1-\delta) \cos \theta_{t} - r_{t} \geq 1 - \epsilon/2 - r_t
\end{align}

and the RHS is always non-negative.

Dividing by the RHS and integrating from $T_1$ to $t$ gives

\begin{align}
    - \log (1 - \epsilon/2 - r_{t}) + \log (1 - \epsilon/2- r_{T_1}) \geq t - T_1
\end{align}

Rearranging gives

\begin{align}
    r_t \geq 1 - \epsilon/2 - (1 - \epsilon/2 - r_{T_1}) e^{T_1 - t}
\end{align}

Note that by definition of $T_2$, it follows that

\begin{align}
    1 - r_t \leq \epsilon/2 + e^{T_1 - t}
\end{align}

So it follows that $T_2 \leq T_1 + \log \frac{2}{\epsilon}$.

Altogether, the total time to achieve $\epsilon$ optimality for all three variables is $O\left(\log \frac{1}{\epsilon} \right)$.

\subsection{Case $s > 1$}

In this case, because we cannot straightforwardly solve or bound the system of ODEs, we need to control rates in stages.  We have a stopping time for one variable at a time, and use local monotonicity to ensure bounds on the remaining variables.

\paragraph{First Phase}

In the first stage, we consider the duration of time $T_1 = \inf\{t \geq 0: v_t \leq v^*\}$ where $v^* := 2^{-s} 6^{-2} s^{-2} r_0^4$, and bound the behavior of each variable.  Below, we will consider $t \in [0, T_1]$.

To control the behavior or $r$, we consider the following manipulations:
\begin{align}
    \frac{d}{dt} \log r^2 &= 2 (1-\delta) r^{s-2} \cos s \theta - 2(v+r^2)^{s-1} \\
    \frac{d}{dt} \log v &= 2\delta \frac{r^s \cos s\theta}{v} - 2(v+r^2)^{s-1}
\end{align}

This implies

\begin{align}
    \frac{d}{dt} \log \frac{r^2}{v} = 2 r^{s-2} \cos s\theta \left(1 - \delta - \delta \frac{r^2}{v} \right)
\end{align}

By definition, in this range of $t$ we have $v_t > \frac{\delta}{1-\delta}$, so it follows that the RHS of this equation is always positive.  Hence it follows that $\log \frac{r^2}{v}$ is increasing, and by monotonicity of $\log$, we have

\begin{align}
    \frac{r^2}{v} \geq \frac{r_0^2}{v_0} \geq r_0^2
\end{align}

This implies that

\begin{align}
    \dot{r} &= (1-\delta) r^{s-1} \cos s \theta -(v+r^2)^{s-1} r\\
    &\geq (1-\delta) r^{s-1} \cos s \theta - \left(\frac{r^2}{r_0^2} + r^2\right)^{s-1}r \\
    & \geq r^{s-1} \left((1-\delta) \cos s \theta - \left(\frac{1}{r_0^2} + 1 \right)^{s-1}r^s \right)
\end{align}

Suppose it is true that $r \leq \frac{1}{6} r_0^2$, then it follows that:

\begin{align}
    r &\leq \frac{r_0^2 (1-\delta) \cos s\theta_0}{2}\\
    &\leq \frac{r_0^2 (1-\delta) \cos s\theta_0}{r_0^2 + 1}\\
    &= \frac{(1-\delta) \cos s\theta_0}{\frac{1}{r_0^2} + 1}\\
    &\leq \frac{\left((1-\delta) \cos s\theta\right)^{1/s}}{\left(\frac{1}{r_0^2} + 1\right)^{\frac{s-1}{s}}}
\end{align}

So it follows that $\dot{r}$ will be positive whenever $r \leq \frac{1}{6} r_0^2$.  We have $r_0 \geq \frac{1}{6} r_0^2$, it follows that $r_t \geq \frac{1}{6} r_0^2$ for $t \leq T_1$.

Finally we can control $v$ by observing that, for $t \in [0, T_1]$, $v \geq v^* \geq (2\delta)^{1/s}$.  Hence,

\begin{align}
    \dot{v} \leq 2 \delta - 2v^s \leq -v^s
\end{align}

which implies

\begin{align}
    -\frac{\dot{v}}{v^s} \geq 1
\end{align}

And integrating from 0 to $t \leq T_1$ gives

\begin{align}
    v_t^{-(s-1)} \geq \frac{1}{s-1} v_t^{-(s-1)} - \frac{1}{s-1} v_0^{-(s-1)} \geq t
\end{align}

Rearranging gives

\begin{align}
    v_t \leq t^{-\frac{1}{s-1}}
\end{align}

This gives a bound on $T_1 \leq (v^*)^{-(s-1)} = O(2^{s^2} r_0^{-4s})$

Lastly by monotonicity we have $\cos s \theta_{T_1} \geq \cos s \theta_0$.

So to summarize:

\begin{align}
    r_{T_1} &\geq \frac{1}{6} r_0^2\\
    \cos s \theta_{T_1} &\geq \cos s \theta_0\\
    v_{T_1} &\leq v^*
\end{align}

Furthermore, we've actually proven that $v_t \leq v^*$ for all $t \geq T_1$, which we will use in subsequent phases.

\paragraph{Second Phase}

We define $T_2 = \inf\{t \geq T_1: r_t \geq 1/5\}$.  As before, if $r_{T_1} \geq 1/5$ then $T_2 = 0$ and we can skip to the next phase, so we assume otherwise.

Using the identity $(1+x)^k \leq 1 + 2^kx$ which holds for any $x \in [0,1]$ and $k\geq 1$, observe that the ODE governing $r$ can now be bounded as:

\begin{align}
    \dot{r} &= (1-\delta) \cos s \theta r^{s-1} - (v + r^2)^{s-1} r \\
    &\geq (1-\delta) \cos s \theta_0 r^{s-1} - \left( \frac{v}{r^2} + 1 \right)^{s-1} r^{2s-1} \\
    &\geq (1-\delta) \cos s \theta_0 r^{s-1} - \left(1 + 2^{s-1} \frac{v}{r^2} \right) r^{2s-1}\\
    &\geq \frac{1-\delta}{2} r^{s-1} - \left(1 + \frac{r_0^4}{2s^2(6r)^2} \right) r^{2s-1}
\end{align}

where in the last step we use that $v \leq v^*$ and plug in the definition of $v^*$ and the bound $\cos s \theta_0 \geq 1/2$.

Consider any $t$ when $r = \frac{1}{6}r_0^2$, and observe that the above inequality implies $\dot{r} > 0$.  Because $r_{T_1} \geq \frac{1}{6}r_0^2$, this implies we will always have $r \geq \frac{1}{6}r_0^2$ for larger values of $t$, and we may bound:

\begin{align}
    \dot{r} &\geq \frac{1-\delta}{2} r^{s-1} - \left(1 + \frac{1}{2s^2} \right) r^{2s-1}
\end{align}

Hence, we can apply Lemma~\ref{lem:r_dynamics} with $a = (1-\delta)/2$, $b = 1 + \frac{1}{2s^2}$, where $k^2 = (a/b)^2 \geq 1/5$, and using the initialization of $r_{T_1}$.  This grants the bound that $T_2 \leq T_1 + O(s^4 r_{T_1}^{-s+1}) = T_1 + O(6^s r_0^{-2s+2})$.

Therefore the new summary is:

\begin{align}
    r_{T_2} &\geq 1/5\\
    \cos s \theta_{T_2} &\geq \cos s \theta_0\\
    v_{T_2} &\leq v^*
\end{align}

\paragraph{Third Phase}

We define $T_3 = \inf\{t \geq T_2: \cos s \theta_t \geq \frac{1 - \frac{1}{4s^4}}{1-\delta} \}$

First of all, note that the bound on $r$ derived in the last phase required lower bounding $\cos s \theta$ by $\cos s \theta_0$.  Since $\cos s \theta$ is non-decreasing, that bound is still true by an identical argument.

So we can bound the ODE for $\theta$:

\begin{align}
    \frac{d}{dt} \cos s \theta &= (1-\delta) sr^{s-2} (1 - \cos^2 s\theta)\\
    &\geq (1-\delta) s(1/5)^{s-2} (1 - \cos^2 s\theta)
\end{align}

Note that by lemma~\ref{lem:theta_dynamics} with $k = (1-\delta) s(1/5)^{s-2}$, we have

\begin{align}
    T_3 \leq T_2 + O(5^s \log s)
\end{align}

The bound $v \leq v^*$ continues to hold.
In summary, we now have:
\begin{align}
    r_{T_3} &\geq 1/5\\
    \cos s \theta_{T_3} &\geq \frac{1 - \frac{1}{4s^4}}{1-\delta}\\
    v_{T_3} &\leq v^*
\end{align}

\paragraph{Fourth Phase}

We define $T_4 = \inf\{t \geq T_3: r_t \geq r^*\}$ where $r^* := 1 - \frac{1}{s^2}$.  Again, consider the non-trivial case where $T_4 \neq 0$.

Because the bound on $v$ is the same, and the bound on $\cos s \theta$ is better than before, we can now bound the ODE of $r$ similarly to the second phase:

\begin{align}
    \dot{r} &\geq \left(1 - \frac{1}{4s^4}\right)  r^{s-1} - \left(1 + \frac{1}{2s^2} \right) r^{2s-1}
\end{align}

Applying Lemma~\ref{lem:r_dynamics} with $k = \frac{1 - \frac{1}{4s^4}}{1 + \frac{1}{2s^2}} = 1 - \frac{1}{2s^2}$, we have:

\begin{align}
    T = \inf \{t \geq T_3: r \geq k^2\} \leq T_3 + O(5^s \log s)
\end{align}

Finally, note that $k^2 = \left(1 - \frac{1}{2s^2}\right)^2 \geq 1 - \frac{1}{s^2}$, which implies that $T_4 \leq T$.

Thus we have:

\begin{align}
    r_{T_4} &\geq r^*\\
    \cos s \theta_{T_4} &\geq \frac{1 - \frac{1}{4s^4}}{1-\delta}\\
    v_{T_4} &\leq v^*
\end{align}

\paragraph{Fifth Phase}

We define $T_5 = \inf\{t \geq T_4: \cos s \theta_t \geq \frac{1-\epsilon/2}{1-\delta}, v_t \leq v^\dag\}$ where $v^\dag = 2^{-s}(\epsilon/2)(r^*)^2$.

Again, since $\cos s \theta$ is increasing and $v$ is always less than $v^*$, the bound on $r \geq r^*$ established in the last step will stay true.

Thus, by the identity $r^k \geq (r^*)^k = \left(1 - \frac{1}{s^2} \right)^k \geq 1 - \frac{k}{s^2}$ we have the ODE inequalities:

\begin{align}
    \frac{d}{dt} \cos s \theta &= (1-\delta) sr^{s-2} (1 - \cos^2 s\theta)\\
    &\geq (1-\delta) s \left(1 - \frac{1}{s} \right)(1 - \cos^2 s \theta)\\
    \dot{v} &= 2\delta r^s \cos s \theta - 2 (v + r^2)^{s-1}v\\
    &\leq 2 \delta - 2 \left(1 - \frac{2(s-1)}{s^2} \right)v
\end{align}

It is easy to see that we'll have the bound

\begin{align}
    T_5 \leq T_4 + O\left(\log \frac{1}{\epsilon}\right)
\end{align}

and in summary

\begin{align}
    r_{T_5} &\geq r^*\\
    \cos s \theta_{T_5} &\geq \frac{1-\epsilon/2}{1-\delta}\\
    v_{T_5} &\leq v^\dag
\end{align}

\paragraph{Sixth Phase}

We define $T_6 = \inf\{t \geq T_5: r_t \geq 1-\epsilon\}$, and assume the non-trivial setting where $T_6 \neq 0$.

Note that $\dot{v}$ is negative when $v = v^\dag$, so the bound $v \leq v^\dag$ remains true for $t \geq T_5$.  Thus, we can control the ODE of $r$ one more time:

\begin{align}
    \dot{r} &= (1-\delta) r^{s-1} \cos s \theta -(v+r^2)^{s-1} r\\
    &\geq (1-\delta) r^{s-1} (1-\epsilon/2) -\left(1 + \frac{v}{r^2}\right)^{s-1} r\\
    &\geq (1-\epsilon/2) r^{s-1} - \left(1 + 2^s \frac{v^\dag}{r^2} \right) r^{2s-1}\\
    &\geq (1-\epsilon/2) r^{s-1} - \left(1 + \epsilon/2 \frac{(r^*)^2}{r^2} \right) r^{2s-1}
\end{align}

One can confirm that when $r = r^*$, the RHS of the above inequality is positive, so $\dot{r} \geq 0$.  Thus, since $r_{T_5} \geq r^*$, it will always be the case that $r \geq r^*$ for $t \geq T_5$, so as before we bound:

\begin{align}
    \dot{r} &\geq (1-\epsilon/2) r^{s-1} - (1+\epsilon/2) r^{2s-1}
\end{align}

By Lemma~\ref{lem:r_dynamics}, we have that

\begin{align}
    T_6 \leq T_5 + O\left(\log \frac{1}{\epsilon}\right)
\end{align}

and thus we've achieved $\epsilon$ optimality for all three of our variables.

\section{Proof of Lemma~\ref{lem:projection}}

Remind from Proposition~\ref{prop:A_matrix} that $A \in \mathbb{C}^{\infty \times M}$ is of the form

\begin{align}
    A_{km} = c_k \sqrt{k} a_m^k
\end{align}

where we assume $c_k > 0$, and $a_m \sim \S$.  Note that

\begin{align}
    1 - \|P_A h^*\|^2 &= \|P_A^\perp h^*\|^2 \\
    &= \min_{w} \|A w - h^*\|^2 \\
\end{align}

so we need to choose a candidate value of $w$.

Consider the block decomposition

\begin{align}
    A = \left[\begin{array}{ c }
    B \\
    \hline
    C
  \end{array}\right]
\end{align}

where $B \in \mathbb{C}^{N \times M}$ and $C \in \mathbb{C}^{\infty \times M}$.  Suppose we decompose $h^* = \left[\begin{array}{ c }
    u \\
    \hline
    0
  \end{array}\right]$ where $u \in \mathbb{C}^N$.  Then if we apply the  pseudoinverse and define $w =  B^+ u$, observe:

\begin{align}
    Aw &=
    \left[\begin{array}{ c }
    B \\
    \hline
    C
  \end{array}\right] B^+u\\
  &=     \left[\begin{array}{ c }
    BB^+u \\
    \hline
    CB^+u
  \end{array}\right]
\end{align}

Observe that we can decompose $B = DX$ where $D$ is a diagonal matrix such that $D_{kk} = c_k\sqrt{k}$ and $X_{km} = a_m^k$.  Since $N < M$, one can see $X$ is a rectangular Vandermonde matrix evaluated on $\{a_m\}_{m=1}^M$.  Almost surely, these values are all pairwise distinct, which implies that $X$ has linearly independent rows.  Since $D$ is diagonal with no zeros along the diagonal, $B$ also has linearly independent rows.  This condition implies $BB^+ = I$.  So we have

\begin{align}
    Aw &=     \left[\begin{array}{ c }
    u \\
    \hline
    CB^+u
  \end{array}\right]
\end{align}

Remember $\|u\| = \|h^*\| = 1$, as $u$ is the first $N$ elements of $h^*$ and hence still only supported on the first $\sqrt{N}$ elements.  Because $B^+ = X^+ D^{-1}$, we have:

\begin{align}
    \|CB^+u\| &\leq \|C\| \|X^+\|  \|D^{-1} u\| \\
\end{align}

We can now go about bounding these norms.

Since $u$ is only supported on the first $\sqrt{N}$ elements and $\|u\| = 1$, it follows $\|D^{-1}u\| \leq \max_{1 \leq k \leq \sqrt{N}} \left|\frac{1}{c_k \sqrt{k}}\right| = \frac{1}{\sigma_-}$.

By Lemma~\ref{lem:X_singular}, we have the bound

\begin{align}
\|X^+\| \leq O\left(\frac{1}{\sqrt{M}}\right)
\end{align}

Finally for any $\hat{w} \in \mathbb{C}^M$ with $\|\hat{w}\| = 1$, we have by Cauchy-Schwarz:

\begin{align}
    \|Cw\|^2 &= \sum_{k=N+1}^\infty \left|\sum_{m=1}^M \hat{w}_m c_k \sqrt{k} a_m^k\right|^2 \\
    &\leq \sum_{k=N+1}^\infty \|\hat{w}\|^2 \sum_{m=1}^M \left|c_k \sqrt{k}\right|^2 \\
    &= M \sum_{k=N+1}^\infty k|c_k|^2\\
    &\leq Me^{-O(\sqrt{N})}
\end{align}

where we use in the last step Assumption~\ref{ass:act}.

With these bounds, we clearly have

\begin{align}
    1 - \|P_A h^*\| &\leq \|Aw - h^*\|^2 \\
    &= \left\|\left[\begin{array}{ c }
    u \\
    \hline
    CB^+u
  \end{array}\right] - \left[\begin{array}{ c }
    u \\
    \hline
    0
  \end{array}\right] \right\|^2 \\
    &\leq \|CB^+u\|^2 \\
    &\leq \frac{M}{\sqrt{M}\sigma_-} e^{-O(\sqrt{N})}
\end{align}

Because $M = O(N^3)$, and we've assumed $1/\sigma_-$ is polynomial in $N$, this bound can be written as $e^{-O(\sqrt{N})}$ for possibly different constants in the big $O$ notation.

\section{Proof of Lemma~\ref{lem:init_concentration}}

    Remind that $m_0 = \langle A w_0, h^* \rangle = \frac{1}{\|Aw\|} \langle A w, h^* \rangle$.  Because the complex Gaussian is invariant to multiplication by an unit modulus complex number, it follows that $\theta_0$ is independent of $r_0$ and uniformly distributed on $\S$.  Because $s$ is a positive integer, $s\theta_0$ is also uniformly distributed on $\S$, and hence $P(\cos s \theta_0 \geq 1/2) = 1/3$.  And by our choice of normalization, $v_0 = 1-r_0^2$ automatically.  So it only remains to prove the first statement is true with high probability.

    We remind that $r_0 = \frac{|\langle A w, h^* \rangle|}{\|Aw\|}$.  By Cauchy-Schwartz, it's clear that $r_0 \leq 1$, so only the lower bound is non-trivial.  If we use the same notation to decompose the matrix $A$ as in the proof of Lemma~\ref{lem:projection}, it's clear that 

    \begin{align}
        |\langle A w, h^* \rangle| &= |\langle Bw, u \rangle|\\
        &= |\langle w, B^\dag u \rangle|
    \end{align}

    If we condition on $B$, then by rotation invariance of the Gaussian, note that $|\langle w, B^\dag u \rangle|$ is distributed identically to $|g|\|B^\dag u \|$ where $g$ is sampled from a one dimensional complex Gaussian.

    By the argument in Lemma~\ref{lem:init_concentration}, since $u$ is only supported on the first $\sqrt{N}$ elements, note that:

    \begin{align}
        \|B^\dag u\| &= \|X^\dag D^\dag u\|\\
        &\geq \sigma_N(X) \|D^\dag u\|\\
        &\geq \sigma_N(X) \sigma_- \\
        &\geq \sigma_- O(\sqrt{M})
    \end{align}

    with probability $1 - 2\exp(-O(N))$ by Lemma~\ref{lem:X_singular}

    Lastly, we need to control 

    \begin{align}
        \|Aw\| \leq \|Bw\| + \|Cw\| \leq (\|B\| + \|C\|)\|w\|
    \end{align}

    And we can write again by Lemma~\ref{lem:X_singular}, with similarly high probability:

    \begin{align}
        \|B\| &= \|DX\| \\
        &\leq \|D\| \|X\|\\
        &\leq \sigma_+ \sigma_1(X) \\
        &\leq \sigma_+ O(\sqrt{M})
    \end{align}

    Combining this with the bound on $\|C\|$ we derived in Lemma~\ref{lem:init_concentration}, and the concentration on $\|w\|$ from Lemma~\ref{lem:complex_gaussian} we have with probability $1 - 2\exp(-O(N))$:

    \begin{align}
        \|Aw\| \leq \left(\sigma_+ O(\sqrt{M}) + e^{-O(\sqrt{N})}\right) O(\sqrt{M})
    \end{align}

    Finally we can say that with probability $1 - 2\exp(-O(N))$

    \begin{align}
        r_0 \geq c\frac{\sigma_-}{\sigma_+\sqrt{M}}
    \end{align}

    for some universal constant $c$.

\section{Auxiliary Lemmas}

\subsection{Dynamics Inequality Lemmas}

The following lemmas provide bounds on our dynamics that we can apply multiple times in different phases of the proof.  Both of these lemmas are essentially special cases of the Bihari-LaSalle Inequality~\citep{bihari1956generalization}, but because the proofs are much simplified due to our setting, and for completeness, we include the proofs below.

\begin{lemma}\label{lem:theta_dynamics}
    Consider $\theta$ with the differential inequality

    \begin{align}
    \frac{d}{dt} \cos s \theta &\geq k (1 - \cos^2 s\theta)
    \end{align}

    with $\cos s \theta_0 \geq 1/2$.  Then we have

    \begin{align}
        \cos s \theta_t \geq \tanh(kt)
    \end{align}

    and hence if $T = \inf\{t\geq 0 : \cos s \theta_t \geq c\}$, then $T \leq \frac{1}{2k} \log \frac{2}{1 - c}$.
\end{lemma}
\begin{proof}
    Clearly the RHS of the inequality is always positive, so we may write:

    \begin{align}
        \frac{\frac{d}{dt} \cos s \theta}{1 - \cos^2 s \theta} \geq k
    \end{align}

    and integrating from $0$ to $t$ gives

    \begin{align}
        \tanh^{-1}(\cos s \theta_t) - \tanh^{-1}(\cos s \theta_0) \geq kt
    \end{align}

    Note $\tanh^{-1}(\cos s \theta_0) \geq 0$, so $\cos s \theta_t \geq \tanh(kt)$.  Since $\cos s \theta_t$ is increasing, it follows that

    \begin{align}
        T \leq \frac{\tanh^{-1}(c)}{k}
    \end{align}
    
    And using the closed form of $\tanh^{-1}$(c) for $|c| < 1$ implies

    \begin{align}
        T &\leq \frac{1}{2k} \log \frac{1 + c}{1 - c} \\
        &\leq  \frac{1}{2k} \log \frac{2}{1 - c}
    \end{align}
\end{proof}
    
\begin{lemma}\label{lem:r_dynamics}
    Consider $s \geq 2$.  Suppose we have constants $0 < a < b$ and a function $r$ of time $t$ with differential identity:

    \begin{align}
        \dot{r} \geq ar^{s-1} - br^{2s-1}
    \end{align}

    Furthermore, assume $0 < r_0$ and it always the case that $r \leq 1$.

    Let $k = \frac{a}{b}$, and $T = \inf\{t \geq 0: r \geq k^2\}$, then:

    \begin{align}
        T \leq \frac{1}{bk^2} \left(\frac{2k}{r_0^{s-1}} + \log \frac{1}{1-k} \right)
    \end{align}
\end{lemma}

\begin{proof}
If $r_0 \geq k^2$, then $T = 0$ and the bound is obviously true.  So assume $r_0 < k^2 \leq k^\frac{1}{s-1}$, where the second inequality follows from the facts that $k < 1$ and $s \geq 2$.

Consider the change of variables $y = r^{s-1}/k$:

\begin{align}
    \dot{y} &= \frac{1}{k}(s-1)r^{s-2} \dot{r} \\
    &\geq \frac{1}{k}(s-1)(ar^{2s-3} - br^{3s-3})\\
    &\geq \frac{1}{k}(ar^{2s-2} - br^{3s-3})\\
    &= \frac{b}{k} (kr^{2s-2} - r^{3s-3})\\
    &= \frac{b}{k}(k^3y^2 - k^3 y^3)\\
    &= bk^2 y^2 (1-y)
\end{align}

For $t \in [0,T)$, the RHS will always be positive, so we can write

\begin{align}
    \frac{\dot{y}}{y^2(1-y)} \geq bk^2
\end{align}

Simple algebra lets us rewrite:

\begin{align}
    \frac{\dot{y}}{y} + \frac{\dot{y}}{y^2} + \frac{\dot{y}}{1 - y} \geq bk^2
\end{align}

And integrating from $0$ to $t$ gives

\begin{align}
    \log y_t - \log y_0 - \frac{1}{y_t} + \frac{1}{y_0} - \log (1-y_t) + \log (1 - y_0) \geq bk^2t
\end{align}

Remind that $\frac{1}{y_t} > 0$ and collecting terms, we have:

\begin{align}
    -\log\left(\frac{1}{y_t} - 1 \right) + \log\left(\frac{1}{y_0} - 1 \right) \geq bk^2t - \frac{1}{y_0}
\end{align}

Taking exponentials and simple bounds:

\begin{align}
    \frac{1}{y_t} - 1 \leq \frac{1}{y_0} \exp(-bk^2t + \frac{1}{y_0})
\end{align}

Rearranging and reminding $y_t = r_t^{s-1}/k$

\begin{align}\label{eq:r_goal}
    \frac{k}{1 + \frac{1}{y_0} \exp(-bk^2t + \frac{1}{y_0})} \leq r_t^{s-1} \leq r_t
\end{align}

To finish the proof, we'll show that $r_t \geq k^2$ is implied by a condition on $t$.  Suppose that

\begin{align}
    t \geq \frac{1}{bk^2} \left(\frac{2}{y_0} + \log \frac{1}{1-k} \right)
\end{align}

Then using the fact that $k < 1$, and $\log x < x$ for all $x > 0$, it follows

\begin{align}
    t &\geq \frac{1}{bk^2} \left(\frac{1}{y_0} +  \log \frac{1}{y_0} + \log \frac{k}{1-k} \right)\\
    &\geq \frac{1}{bk^2} \left(\frac{1}{y_0} +  \log \frac{1}{y_0\left(\frac{1}{k} - 1\right)} \right)
\end{align}

Rearranging implies that

\begin{align}
k \leq \frac{1}{1 + \frac{1}{y_0} \exp(-bk^2t + \frac{1}{y_0})}
\end{align}

and plugging this into Equation~\ref{eq:r_goal} implies that $r_t \geq k^2$.  Hence, the stopping time $T$ obeys:

\begin{align}
    T \leq \frac{1}{bk^2} \left(\frac{2}{y_0} + \log \frac{1}{1-k} \right)
\end{align}

Plugging in the definition of $y_0$ gives the bound.
\end{proof}

\subsection{Concentration Inequality Lemmas}

We require a few very standard lemmas, adapting concentration inequalities to the complex setting.

\begin{lemma}\label{lem:complex_gaussian}
    If $w$ is drawn from the standard complex Gaussian on $M$ dimensions, then

    \begin{align}
        P(|\|w\| - \sqrt{M} | \geq t) \leq 2 \exp(-ct^2)
    \end{align}

    for some universal constant $c$.
\end{lemma}
\begin{proof}
    Note that an equivalent way of sampling a complex Gaussian is $w = \frac{1}{\sqrt{2}} (w_R + i w_C)$ with $w_R, w_C$ both sampled iid from a standard real Gaussian on $M$ variables.  Therefore

    \begin{align}
        \|w\|^2 &= \frac{\|w_R\|^2 + \|w_C\|^2}{2}\\
        &= \frac{1}{2} \left\| \left[\begin{array}{ c }
    w_R\\
    \hline
    w_C
  \end{array}\right]\right\|^2
    \end{align}

    Note that $\hat{w} := \left[\begin{array}{ c }
    w_R\\
    \hline
    w_C
  \end{array}\right]$ is simply a standard Gaussian on $2M$ variables, so from Theorem 3.1.1 in~\citet{vershynin2018high}:

  \begin{align}
      P(|\|w\| - \sqrt{M} | \geq t) &= P(|\|\hat{w}\| - \sqrt{2M} | \geq t\sqrt{2})\\
      &\leq 2 \exp(-ct^2)
  \end{align}

  for some universal constant $c$.
    
\end{proof}

\begin{lemma}\label{lem:X_singular}
    Let $a_m \sim \S$ be sampled iid, for $m = 1, \dots, M$, and define $X \in \mathbb{C}^{N \times M}$ as $X_{nm} = a_m^n$.  Then if we choose $M = O(N^3)$, with probability $1 - 2\exp(-O(N))$:

    \begin{align}
        \sigma_1(X) = \Theta(\sqrt{M}), \sigma_N(X) = \Theta(\sqrt{M})
    \end{align}
\end{lemma}
\begin{proof}
    Note that the columns of $X$ are independent, mean zero, and isotropic.  Let $X_m$ be the $m$th column, and consider any $v \in \mathbb{C}^{M}$ with $\|v\| = 1$.  Note that $\|X_m\| = \sqrt{N}$, so it follows that

\begin{align}
    \|\langle X_m, v \rangle\|_{\psi_2} \leq \sqrt{N}
\end{align}

where $\|\cdot\|_{\psi_2}$ denotes the subgaussian norm~\citep{vershynin2018high}.  Hence, we can apply Theorem 4.6.1 from~\citet{vershynin2018high} to $X^T$.  Note, although this proof assumes real-valued variables, the same arguments follow through with no change to complex variables given the subgaussian bound on $\|\langle X_m, v \rangle\|_{\psi_2}$.  Hence,

\begin{align}
    \sqrt{M} - cN(\sqrt{N} + t) \leq \sigma_N(X) \leq \sigma_1(X) \leq \sqrt{M} + CN(\sqrt{N} + t)
\end{align}

for universal constants $c, C$ and with probability $1 - 2\exp(-t^2)$.  Choosing $t = \sqrt{N}$ and $M = O(N^3)$ gives the result.
    
\end{proof}

\subsection{Valid Activations}

We quickly note a one simple choice of many possible activation functions that meets our criteria in Assumption~\ref{ass:act}.

\begin{proposition}\label{prop:arctan}
    Let $\sigma(z) = \arctan{\xi z} + \xi z \arctan{\xi z}$ for $\xi = 1 - \frac{1}{\sqrt{N}}$.  Then this activation satisfies Assumption~\ref{ass:act}, and $\sigma_+ \leq \sqrt{2}$, $\sigma_- \geq O\left(\frac{1}{N^{1/4}}\right)$.
\end{proposition}
\begin{proof}
    Observe that $\sigma$ is analytic on the unit disk following from properties of $\arctan$, with the Laurent series

    \begin{align}
        \sigma(z) = \xi z + \xi^2 z^2 - \frac{\xi^3}{3} z^3 - \frac{\xi^4}{3} z^4 + \dots
    \end{align}

    So the only coeffient equal to zero is the constant term.  Moreover, if split into sequence of odd degree and even degree coefficients, both sequences are decreasing in absolute value, so we can instantly say that $\sigma_+ \leq \sqrt{2}$ and 
    
    \begin{align}
    \sigma_- &= \min_{1 \leq k \leq \sqrt{N}} |c_k|\sqrt{k} \\
    &\geq \frac{\left(1 - \frac{1}{\sqrt{N}}\right)^{\sqrt{N}}}{\sqrt{N}} N^{1/4} = O\left(\frac{1}{N^{1/4}}\right)
    \end{align}

    Moreover, we can calculate:

    \begin{align}
        \sum_{k=N+1}^\infty k|c_k|^2 &\leq \sum_{k=N+1}^\infty \frac{k\xi^{k-1}}{(k-1)^2} \\
        &\leq \sum_{k=N+1}^\infty \xi^{k-1} \\
        &\leq \frac{\xi^N}{1-\xi}\\
        &\leq e^{-O(\sqrt{N})}
    \end{align}
\end{proof}

\end{document}